\documentclass[sigconf,natbib=true,anonymous=False]{acmart}
\usepackage{CJKutf8}
\AtBeginDocument{%
  \providecommand\BibTeX{{%
    \normalfont B\kern-0.5em{\scshape i\kern-0.25em b}\kern-0.8em\TeX}}}

\usepackage{multirow} 
\usepackage{hyperref}
\usepackage{tikz}
\usetikzlibrary{fit, backgrounds, positioning, } 
\usepackage{tcolorbox}

\usepackage{amsmath,amsthm,amsfonts}
\newtheorem{asmp}{Assumption}
\newtheorem{lemma}{Lemma}

\newcommand{\delequal}{\stackrel{\Delta}{=}}

\begin{document}
\title{Answering Causal Queries at Layer 3 with DiscoSCMs: Embracing Heterogeneity}


\author{Heyang Gong}
\affiliation{%
  \institution{Kuaishou Inc.}
  \city{Beijing}
  \country{China}
}
\email{gongheyang03@kuaishou.com}

\renewcommand{\shortauthors}{Heyang Gong}

\begin{abstract} 
In the realm of causal inference, Potential Outcomes (PO) and Structural Causal Models (SCM) are recognized as the principal frameworks. However, when it comes to addressing Layer 3 valuations—counterfactual queries deeply entwined with individual-level semantics—both frameworks encounter limitations due to the degenerative issues brought forth by the consistency rule.  This paper advocates for the Distribution-consistency Structural Causal Models (DiscoSCM) framework as a pioneering approach to counterfactual inference, skillfully integrating the strengths of both PO and SCM. The DiscoSCM framework distinctively incorporates a unit selection variable $U$ and embraces the concept of uncontrollable exogenous noise realization. Through personalized incentive scenarios, we demonstrate the inadequacies of PO and SCM frameworks in representing the probability of a user being a complier (a Layer 3 event) without degeneration, an issue adeptly resolved by adopting the assumption of independent counterfactual noises within DiscoSCM. This innovative assumption broadens the foundational counterfactual theory, facilitating the extension of numerous theoretical results regarding the probability of causation to an individual granularity level and leading to a comprehensive set of theories on heterogeneous counterfactual bounds. Ultimately, our paper posits that if one acknowledges and wishes to leverage the ubiquitous heterogeneity, understanding causality as invariance across heterogeneous units, then DiscoSCM stands as a significant advancement in the methodology of counterfactual inference. 
\end{abstract}




\maketitle
\begin{CJK}{UTF8}{gbsn} 

\section{Introduction}

Causal inference, a cornerstone of empirical research, is underpinned by two principal frameworks: Potential Outcomes (PO) \citep{rubin1974estimating, neyman1923application} and Structural Causal Models (SCMs) \citep{pearl2009causality}. The PO framework, initiated with a focus on experimental units, associates each treatment value $T=t$ with a corresponding set of potential outcomes $Y(t)$, of which only the one aligned with the received treatment is observable. In contrast, the SCM framework, through structural equations, delves into the causal mechanisms behind observed phenomena, defining potential outcomes as variables within submodels crafted by interventions, thereby operationalized via the $do$-operator. Despite their distinct approaches, both frameworks are largely considered equivalent, allowing for a comprehensive translation of causal statements. These methodologies have found wide applications across disciplines, notably in the scenario of industrial personalizing decision-making with uplift modeling technique that leverages heterogeneous causal effects \cite{gutierrez2017causal, diemert2018large, goldenberg2020free, zhang2021unified, ai2022lbcf}.

Recently, counterfactuals at layer 3 rather than causal effects at layers 2 have been considered as increasingly important in such scenarios \cite{mueller2023personalized, li2019unit, li2022unit}. Considering a pertinent motivating example, the probability of a counterfactual event $P(Y(t)=y, Y(t')=y')$, where $T$ represents the decision to issue a discount coupon and $Y$ the action of making a purchase, has emerged as a focal point \cite{li2019}.  This underscores the critical role of counterfactuals, inherently linked to the joint distribution of counterfactual outcomes, in refining personalized decision-making processes.

The debate surrounding these methodologies, however, persists. Dawid \cite{geneletti2011defining, dawid2023personalised} critically challenges the foundational premises of counterfactual-based decision-making, questioning the real existence of potential responses and labeling such strategies as potentially misleading. In contrast, Gong \cite{gong2024discoscm} underscores the necessity of counterfactual methodologies but points out the limitations of existing frameworks in addressing Layer 3 valuation, particularly due to degeneration issues caused by the consistency rule  \cite{angrist1996identification}. Specifically, for a population of users with observed data ($\{ x_u, t_u, y_u\}_{u \in U}$), the joint distribution $P(y_t, y'_{t'})$ at the individual level or for a sub-population ($t_u=t, y_u=y$) must degenerate, since one component of the joint distribution degenerates to a constant that is consistent with the observed outcome value. To circumvent such degeneration challenges, this paper champions the Distribution-consistency Structural Causal Models (DiscoSCM) framework, where the consistency rule is replaced by a distribution-consistency rule, allowing it to accommodate non-degenerative counterfactuals that are heterogeneous across individuals. 

In this paper, we delve into counterfactual inference, such as determining the probability of being a complier in the aforementioned example. Our investigation on DiscoSCMs reveals the challenge of indeterminable counterfactuals, leading us to adopt the assumption of Independent Counterfactual Noises (ICN) for identification, expanding the foundational counterfactual theory. Furthermore, we found that even this ICN assumption does not hold,  we still can derives heterogeneous bounds across individuals for counterfactuals regarding the probability of causation.  We posit that DiscoSCM is exceptionally promising for counterfactual inference that embraces heterogeneity.

Our investigation into DiscoSCMs and the adoption of the Independent Counterfactual Noises (ICN) assumption for identification expands the foundational theory of counterfactuals. This enables us to establish bounds for heterogeneous counterfactuals, enhancing our understanding of causality at an individual granularity level. We posit that DiscoSCM is exceptionally promising for counterfactual inference that embraces heterogeneity.

Our contributions are manifold:
\begin{itemize}
\item Identifying the indeterminate nature of counterfactuals and proposing the ICN assumption as a solution.
\item Developing and validating novel theorems that clarify the bounds of counterfactuals  at individual-level, thereby enhancing our understanding of \textit{heterogeneous counterfactuals}.
\item Advocating for the DiscoSCM framework as a significant advancement in addressing the ubiquitous heterogeneity, and understanding causality as invariance across heterogeneous units.
\end{itemize}

\section{Formulation of DiscoSCMs}

The mainstream frameworks for causal modeling, including PO and SCM as detailed in Appendix \ref{sec:prelim}, fundamentally rely on the consistency rule. This reliance introduces limitations in the capacity for counterfactual inference \cite{gong2024discoscm}. Consider the following hypothetical scenario: "If an individual with average ability scores exceptionally high on a test due to good fortune, what score would they achieve if they retook the test under identical conditions? Would it be another exceptionally high score or an average one?"   Frameworks anchored in the consistency rule might predict another high score, yet intuitively, an average score seems more plausible, recognizing the non-replicability of good fortune.  To accommodate this ``uncontrollable good fortune'', the distribution-consistency assumption is proposed.
\begin{asmp}[\textbf{Distribution-consistency} \cite{gong2024discoscm}]
\label{assump:distri-consist}
For any individual represented by $U=u$ with an observed treatment $X = x$, the counterfactual outcome $Y(x)$ is equivalent in distribution to the observed outcome $Y$. Formally, 
\begin{equation}
\label{eq:assump:distri-consist}
X = x, U = u \Rightarrow Y(x) \stackrel{d}{=} Y,
\end{equation}
where \(\stackrel{d}{=}\) indicates equivalence in distribution.
\end{asmp}
Differing from Assumption \ref{assump:consist}, key modifications include the inclusion of $U=u$ and the use of $\stackrel{d}{=}$. In alignment with these modifications, the Distribution-consistency Structural Causal Model (DiscoSCM)  \cite{gong2024discoscm} is proposed, notable for incorporating a unit selection variable $U$ and acknowledging the insight of uncontrollable exogenous noise.
\begin{definition}
    \label{def:discoscm}
    A Structural Causal Model (DiscoSCM) is a tuple $\langle  U, \mathbf{E}, \mathbf{V}, \mathcal{F}\rangle$, where
    \begin{itemize}
        \item $U$ is a unit selection variable, where each instantiation $U=u$ denotes an individual. It is associated with a probability function $P(u)$, uniformly distributed by default.
        \item $\mathbf{E}$ is a set of exogenous variables, also called noise variables, determined by factors outside the model. It is independent to $U$ and associated with a probability function $P(\*e)$;
        \item $\mathbf{V}$ is a set of endogenous variables of interest $\{V_1, V_2, \ldots, V_n\}$, determined by other variables in $\mathbf{E} \cup \mathbf{V}$;
        \item $\mathcal{F}$ is a set of functions $\{f_1(\cdot, \cdot; u), f_2(\cdot, \cdot; u), \ldots, f_n(\cdot, \cdot; u)\}$, where each $f_i$ is a mapping from $E_{i} \cup Pa_{i}$ to $V_{i}$, with $E_{i} \subseteq \mathbf{E}$, $Pa_{i} \subseteq \mathbf{V} \setminus V_{i}$, for individual $U=u$. Each function assigns a value to $V_i$ based on a select set of variables in $\mathbf{E} \cup \mathbf{V}$. That is, for $i=1,\ldots,n$, each $f_i(\cdot, \cdot; u) \in \mathcal{F}$ is such that 
        $$v_i \leftarrow f_{i}(pa_{i}, e_i; u),$$ 
        i.e., it assigns a value to $V_i$ that depends on (the values of) a select set of variables in $\*E \cup \*V$ for each individual $U=u$. 
    \end{itemize}
\end{definition}

Intervention logic within the DiscoSCM framework is articulated as follows.
\begin{definition}
    For a DiscoSCM $\langle  U, \mathbf{E}, \mathbf{V}, \mathcal{F}\rangle$, $\*X$ is a set of variables in $\*V$ and $\*x$ represents a realization, the $do(\mathbf{x})$ operator modifies: 1) the set of structural equations $\mathcal{F}$ to 
        \begin{align*}
            \mathcal{F}_{\mathbf{x}} := \{f_i : V_i \notin \mathbf{X}\} \cup \{\mathbf{X} \leftarrow x\},
        \end{align*}
        and; 2) noise $\mathbf{E}$ to couterfactual noise $\mathbf{E}(\mathbf{x})$ maintaining the same distribution. \footnote{Note that $\mathbf{E}(\mathbf{x})$ is not a function of $\*x$, but rather a random variable indexed by $\*x$. Importantly, it shares the same distribution as $\mathbf{E}$.} The induced submodel $\langle U, \mathbf{E}(\*x), \mathbf{V}, \mathcal{F}_{\mathbf{x}}\rangle$ is called the \textit{interventional DiscoSCM}.
\end{definition}
Therefore, we define counterfactual outcomes in counterfactual worlds created by interventions.
\begin{definition}[\textbf{Counterfactual Outcome}]
    For a DiscoSCM $\langle  U, \mathbf{E}, \mathbf{V}, \mathcal{F}\rangle$, $\*X$ is a set of variables in $\*V$ and $\*x$ represents a realization. The counterfactual outcome $\*Y^d(\*x)$ (or denoted as $\*Y(\*x)$, $\*Y_{\*x}(\*e_{\*x})$ when no ambiguity concerns) is defined as the set of variables $\*Y \subseteq \*V$ in the submodel $\langle U, \mathbf{E}(\*x), \mathbf{V}, \mathcal{F}_{\mathbf{x}}\rangle$. In the special case that $\*X$ is an empty set, the corresponding submodel is denoted as $\langle U, \mathbf{E}^d, \mathbf{V}, \mathcal{F}\rangle$ and its counterfactual noise and outcome as $\*E^d$ and $\*Y^d$, respectively.
\end{definition}
Notice that the counterfactual outcome $Y_u^d(t)$ is a random variable that equals in distribution to $Y_u$, rather than a constant $y$, when observing $X_u = x, Y_u = y$ for an individual $u$. In other words, it can conceptually be seen as an extension both PO and SCM frameworks, achieved by replacing the traditional consistency rule with a distribution-consistency rule. Furthermore, the Layer valuations can defined for each individual $U=u$. 
\begin{definition}[\textbf{Layer Valuation with DiscoSCM} \cite{gong2024discoscm}]
\label{def:semantics}
A DiscoSCM $\langle U, \mathbf{E}, \mathbf{V}, \mathcal{F}\rangle$ induces a family of joint distributions over counterfactual outcomes $\*Y(\*x), \ldots, \*Z({\*w})$, for any $\*Y$, $\*Z, \dots, \*X$, $\*W \subseteq \*V$:
\begin{align}\label{eq:def:l3-semantics_new}
    P(\*{y}_{\*{x}},\dots,\*{z}_{\*{w}}; u) =
\sum_{\substack{\{\*e_{\*x}\, \cdots , \*e_{\*w}\;\mid\;\*{Y}^d({\*x})=\*{y},\cdots  \\ , \*{Z}^d({\*w})=\*z, U=u\}}}
    P(\*e_{\*x}, ..., \*e_{\*w}).
\end{align}
is referred to as Layer 3 valuation. In the specific case involving only one intervention \footnote{When \( \*X = \emptyset \), we simplify the notation \( \*Y^d(\*x) \) to \( \*Y^d \) and \( \*E_{\*x} \) to \( \*E^d \).
}, e.g., $do(\*x)$:
\begin{align}
    \label{eq:def:l2-semantics_new}
    P({\*y}_{\*x}; u) = 
    \sum_{\{\*e_{\*x} \;\mid\;  {\*Y}^d({\*x})={\*y}, U=u\}}
    P(\*e_{\*x}),
\end{align}
is referred to as Layer 2 valuation. The case when no intervention:
\begin{align}
    \label{eq:def:l1-semantics_new}
    P({\*y}; u) = 
    \sum_{\{\*e \;\mid\; {\*Y}={\*y}, U=u\}}
    P(\*e),
\end{align}
is referred to as Layer 1 valuation. Here, $\*y$ and $\*z$ represent the observed outcomes, $\*x$ and $\*w$ the observed treatments, $\*e$ the noise instantiation, $u$ the individual, and we denote $\*y_{\*x}$ and $\*z_{\*w}$ as the realization of their corresponding counterfactual outcomes, $\*e_{\*x}$, $\*e_{\*w}$ as the instantiation of their corresponding counterfactual noises.
\end{definition}

This framework introduces a novel lens -- individual/population -- to address causal questions, when climbing the Causal Hierarchy: associational, interventional, and counterfactual layers. Specifically, consider a DiscoSCM where \(e\) represents the observed trace or evidence (e.g., $X = x, Y = y$) \footnote{Please distinguish between the ``evidence'' (\( e\)) and the instantiation of ``exogenous noise '' (\(\*e\)).}, we have:

\begin{theorem}[\textbf{Individual-Level Valuations} \cite{gong2024discoscm}]
For any given individual \(u\),
\label{theo:individual}
    \begin{align*}
        P(y_x|e;u) = P(y_x;u) = P(y|x;u) 
    \end{align*}  
indicating that the (individual-level) probabilities of an outcome at Layer 1/2/3 are equal.
\end{theorem}

Individual-level valuations (e.g. $P(y_x|e;u)$) are primitives while population-level valuations  (e.g. $P(y_x|e)$) are derivations.

\begin{theorem}[\textbf{Population-Level Valuations}  \cite{gong2024discoscm}]
\label{theo:population}
Consider a DiscoSCM wherein $Y(x)$ is the counterfactual outcome, and \(e\) represents the observed trace or evidence. The Layer 3 valuation \(P(Y^d(x)|e)\) is computed through the following process:

\textbf{Step 1 (Abduction):} Derive the posterior distribution \(P(u|e)\) of the unit selection variable \(U\) based on the evidence \(e\).

\textbf{Step 2 (Valuation):} Compute individual-level valuation \(P(y_x;u)\) in Def. \ref{def:semantics} for each unit \(u\).

\textbf{Step 3 (Reduction):} Aggregate these individual-level valuations to obtain the population-level valuation as follows:
\begin{equation}
\label{eq:population}
P(Y^d(x)=y|e) = \sum_u P(y_x;u) P(u|e),
\end{equation}
\end{theorem}

Hence, within this framework, counterfactual queries of the form $P(y_x|e)$ (e.g., Positive Necessity (PN) in Def. \ref{def:pn} and Positive Sufficiency (PS) in Def. \ref{def:ps}), have been effectively addressed, as evidenced by Theorem \ref{theo:individual} and \ref{theo:population}. However, for other forms of counterfactual inquiries, notably the Probability of Necessity and Sufficiency (PNS) as delineated in Def. \ref{def:pns}, the framework has yet to offer conclusive results. This means that the probability of being a complier, represented as $P(y_t, y'_{t'})$ in our motivating example, remains an unresolved issue within the DiscoSCM paradigm. The subsequent section will delve into approaches to tackle this challenge.

\section{Indeterminable Counterfactuals}

We contend that the core of causal modeling stems from bridging observed phenomena (for example, the dataset ${(t_u, y_u)}{u \in U}$) with causal quantities, such as causal graphs, structural equations, or Layer valuations. Thus, in discussions on counterfactual inference, it's pivotal to explore how it connects with the data. Given that we typically have observational data and randomized controlled trial (RCT) data at our disposal, a natural question arises: ``Can we identify $P(y_t, y'_{t'})$ solely from the data?'' Our response is that, in the absence of appropriate assumptions, no amount of data will enable you to identify it. We refer to this predicament as the issue of \textit{indeterminable counterfactuals}, which we will elucidate from various perspectives in the following discussion. 


The primary dimension to consider is that each individual possesses multiple counterfactual outcomes, yet only one observed datapoint is available in the dataset, representing a sample from one of these counterfactual outcomes. \footnote{This viewpoint signifies a departure from traditional frameworks. It's essential to recognize that within the DiscoSCM framework, any counterfactual outcomes in the counterfactual world are unobservable; only outcomes within the factual world manifest an observed trace. However, by adhering to the principle of distribution consistency, this observed trace is interpreted as a sample of one of the counterfactual outcomes. Therefore, the traditional assertion that "only one of the counterfactual outcomes can be observed" in \cite{holland1986statistics} is imprecise within the context of the DiscoSCM framework.} As a result, this singular sample within the dataset does not convey any information about the interrelations among counterfactual outcomes for that particular individual $u$. In essence, this means that the probability of concurrent counterfactual outcomes, represented as $P(y_t, y'_{t'})$, is fundamentally indiscernible.

At this point, counterfactual inference demands the integration of additional information, conditions, or assumptions. For instance, the utilization of Randomized Controlled Trial (RCT) data for estimating the Average Treatment Effect (ATE) or other causal parameters, coupled with the homogeneity of causal parameters across individuals, enables the derivation of individual-level causal parameter estimates. This approach aligns more closely with statistical methodologies. When employing machine learning techniques, it generally necessitates the availability of pre-treatment features for individuals, operating under the implicit assumption that individuals with similar features should exhibit similar outcomes. However, to our knowledge, none of these methodologies have explicitly engaged in the exploration of correlations among counterfactual outcomes at the individual level.

To further elucidate the indeterminable nature of $P(y_t, y'_{t'})$, let's consider an ideal randomized controlled trial (RCT) dataset, assumed to exhibit homogeneous counterfactuals across individuals. Imagine a hypothetical RCT dataset involving 8 users, with their observed outcomes as depicted in Table \ref{tab:unit_outcome}.
\begin{table}[htbp]
\centering
\caption{Samples for Counterfactual outcomes.}
\label{tab:unit_outcome}
\small
\setlength{\tabcolsep}{5pt}
\begin{tabular}{lcc}
\toprule
 & \multicolumn{2}{c}{Counterfactual Outcome $Y^d_u(t)$} \\
\cmidrule{2-3}
Units & Control & Treatment \\
\midrule
1 & 1 & - \\
2 & 0 & - \\
3 & 0 & - \\
4 & 1 & - \\
5 & - & 0 \\
6 & - & 1 \\
7 & - & 1 \\
8 & - & 1 \\
\bottomrule
\end{tabular}
\end{table}
In this dataset, leveraging the principle of homogeneity, the samples of $Y^d(1)$ for the latter four individuals are considered as approximations for the first four individuals' counterfactual outcomes under treatment. Thus, for these first four individuals, we seemingly have access to samples for $(Y^d(0), Y^d(1))$. This setup allows for the deduction that the absolute value of the correlation between $Y^d(0)$ and $Y^d(1)$ should be inferior to the correlation observed between the sequences [0, 0, 1, 1] and [0, 1, 1, 1], based on the premise of uniformity across individuals and their observed outcomes. However, such an analysis does not aid in clear identification, as the observed correlation can manifest as either positive or negative, highlighting the inherently indeterminate nature of deriving such correlations directly from data. More precisely, according to Sklar’s Theorem, the joint distribution information of $Y^d(0)$ and $Y^d(1)$ can be dissected into correlation information represented by a Copula function and marginal distributions information present in the data. Crucially, the information pertaining to the Copula function is not contained within the data. 
To more explicitly illustrate the indeterminate nature of $P(y_t, y'_{t'})$, , let's delve into the following simplest example of a DiscoSCM:
\begin{example}
\label{eg:noise}
Consider a DiscoSCM model for the outcome $Y$, influenced by a binary feature $X$ and a treatment $T$,
\begin{align*}
Y &= E, \\
Y^d(t) &= E_X(t),
\end{align*}
where $E \sim N(0, 1)$ signifies a standard normal distribution, and $E_X(t) \stackrel{d}{=} E$ denotes that $E_X(t)$ is equivalent in distribution to $E$.
\end{example}
This DiscoSCM presents a straightforward solution for counterfactual outcomes $(Y^d(0), Y^d(1))$, modeled as a bivariate normal distribution with standard normal marginals,\footnote{Simulation code is accessible at \url{https://colab.research.google.com/drive/1TcrQzPYd5g3i_hmh0r4TSWP2pbfsHvS9?usp=sharing}.} as illustrated in Fig. \ref{fig:noise_corr}. Consequently, the probability associated with counterfactuals $(Y^d(0), Y^d(1))$ is defined by the correlation between the counterfactual noises $E_X(0)$ and $E_X(1)$. Specifically, setting $E_X(t) = E$ results in a scenario akin to a traditional SCM, where the correlation between $Y(0)$ and $Y(1)$, depicted on the left side of Fig. \ref{fig:noise_corr}, becomes a fixed value of 1. Typically, the correlation between $Y(0)$ and $Y(1)$ varies, exhibiting diverse patterns based on the values of $X$, as shown in the middle subgraph of Fig. \ref{fig:noise_corr}, indicating the presence of heterogeneous counterfactual probabilities dependent on the variable $X$.
\begin{figure}[http]
\centering
\includegraphics[width=0.5\textwidth]{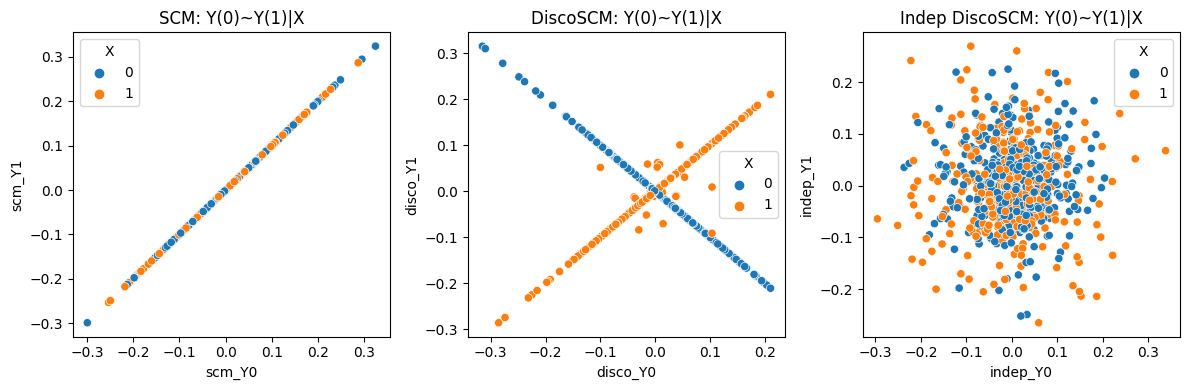}
\caption{The simplistic DiscoSCM from Example \ref{eg:noise} presenting three different correlation patterns.}
\label{fig:noise_corr}
\end{figure}

In conclusion, the probability $P(y_t, y'_{t'})$, depended on the correlation pattern among counterfactual outcomes as illustrated in the middle of Fig. \ref{fig:noise_corr}, remains indeterminate due to the inherent limitations concerning sample availability across different counterfactual worlds.\footnote{This limitation echoes the fundamental problem of causal inference in traditional frameworks, which posits that, for any given individual, at most one potential outcome is observable,  with all the others missing\cite{holland1986statistics}.} Thus, the critical question arises: What cross-world assumptions might we adopt to surmount these limitations? Fortunately, the third scenario, as depicted by the right subgraph of Fig.~\ref{fig:noise_corr}, computationally overcomes the inherent disadvantages of the first two, making DiscoSCM with independent counterfactual noises the main subject of our following section.

\section{DiscoSCM with Independent Counterfactual Noises}

Formally, the DiscoSCM with independent counterfactual noises is defined as follows:
\begin{definition}
\label{def:indepDiscoSCM}
A DiscoSCM $\langle U, \mathbf{E}, \mathbf{V}, \mathcal{F}\rangle$ with independent counterfactual noises induces a family of joint distributions over counterfactual outcomes $\*Y_{\*x}, \ldots, \*Z_{\*w}$, for any $\*Y, \*Z, \dots, \*X, \*W \subseteq \*V$, satisfying:
\begin{align}
     P(\*e_{\*x}, ..., \*e_{\*w}) = P(\*e_{\*x}) \cdots P(\*e_{\*w})
\end{align}
\end{definition}
The term ``independent counterfactual noises'' refers to the independence among the exogenous noises across different counterfactual worlds. Combined with Eq. \eqref{eq:def:l3-semantics_new}, the following theorem for individual-level Layer 3 valuations can be derived.

\begin{theorem}
\label{thm:layer3}
Given a DiscoSCM $\langle U, \mathbf{E}, \mathbf{V}, \mathcal{F}\rangle$ with independent counterfactual noises, the counterfactual outcomes $\mathbf{Y}_{\mathbf{x}}, \ldots, \mathbf{Z}_{\mathbf{w}}$ for any individual $u$ satisfy the following relationship:
\begin{align}
    P(\mathbf{y}_{\mathbf{x}},\dots,\mathbf{z}_{\mathbf{w}}) =  P(\mathbf{y}_{\mathbf{x}})\cdots P(\mathbf{z}_{\mathbf{w}})
\end{align}
\end{theorem}

\begin{proof}
Leveraging Def. \ref{def:semantics} and the independence among counterfactual noises, the joint probability of counterfactual outcomes for an individual $u$:
\begin{align*}
    &P(\mathbf{y}_{\mathbf{x}},\dots,\mathbf{z}_{\mathbf{w}}; u) \\
    &= \sum_{\substack{\{\mathbf{e}_{\mathbf{x}}, \cdots , \mathbf{e}_{\mathbf{w}} \mid \mathbf{Y}^d(\mathbf{x})=\mathbf{y},\cdots, \mathbf{Z}^d(\mathbf{w})=\mathbf{z}, U=u\}}} P(\mathbf{e}_{\mathbf{x}}, \ldots, \mathbf{e}_{\mathbf{w}}) \\
    &= \sum_{\substack{\{\mathbf{e}_{\mathbf{x}}, \cdots , \mathbf{e}_{\mathbf{w}} \mid \mathbf{Y}^d(\mathbf{x})=\mathbf{y},\cdots, \mathbf{Z}^d(\mathbf{w})=\mathbf{z}, U=u\}}} P(\mathbf{e}_{\mathbf{x}}) \cdots P(\mathbf{e}_{\mathbf{w}}) \\
    &= \sum_{\substack{\{\mathbf{e}_{\mathbf{x}}\mid \mathbf{Y}^d(\mathbf{x})=\mathbf{y}, U=u\}}} P(\mathbf{e}_{\mathbf{x}}) \cdots \sum_{\substack{\{\mathbf{e}_{\mathbf{w}}\mid \mathbf{Z}^d(\mathbf{w})=\mathbf{z}, U=u\}}} P(\mathbf{e}_{\mathbf{w}}) \\
    &=  P(\mathbf{y}_{\mathbf{x}}; u)\cdots P(\mathbf{z}_{\mathbf{w}}; u).
\end{align*}
\end{proof}
This theorem simplifies a Layer 3 valuation into the product of Layer 2 valuations, where the latter represents an identifiable and tractable causal question that can be addressed with statistical or machine learning methods. Such a result  seems almost too good to be true, prompting the subsequent sections of this paper to delve deeply into a comprehensive analysis, ensuring the soundness and efficacy of this model. An illustrative example follows:

\begin{example}
\label{eg:discoSCMs}
Consider a DiscoSCM for the outcome $Y$ with features $X_0, X_1, X_2$ and binary treatment $T$:
\begin{align*}
    Y &=   0.5 I[X_0=1] \cdot (T+1) + 0.1 X_2 \cdot E  \\
    Y^d(t) &= 0.5 I[X_0=1] \cdot (t+1) + 0.1 X_2 \cdot E_{X_1}(t)
\end{align*}
where $E, E_{X_1}(t) \sim N(0, 1), t=0, 1$ denote the noise and counterfactual noises respectively. 
\end{example}
As an extension of Example \ref{eg:noise}, Fig. \ref{fig:indepDiscoSCM} showcases an RCT dataset produced by this DiscoSCM. The features $X_0, X_1, X_2$ respectively govern heterogeneous causal effects, counterfactual noises correlations, and consistency probabilities. The top row of the figure displays three unique DiscoSCM models with identical Layer 1 and Layer 2 valuations. Rows two to four illustrate the heterogeneous Layer 3 valuations stemming from different counterfactual noises correlation patterns.
\begin{figure}[http]
    \centering
    \includegraphics[width=0.49\textwidth]{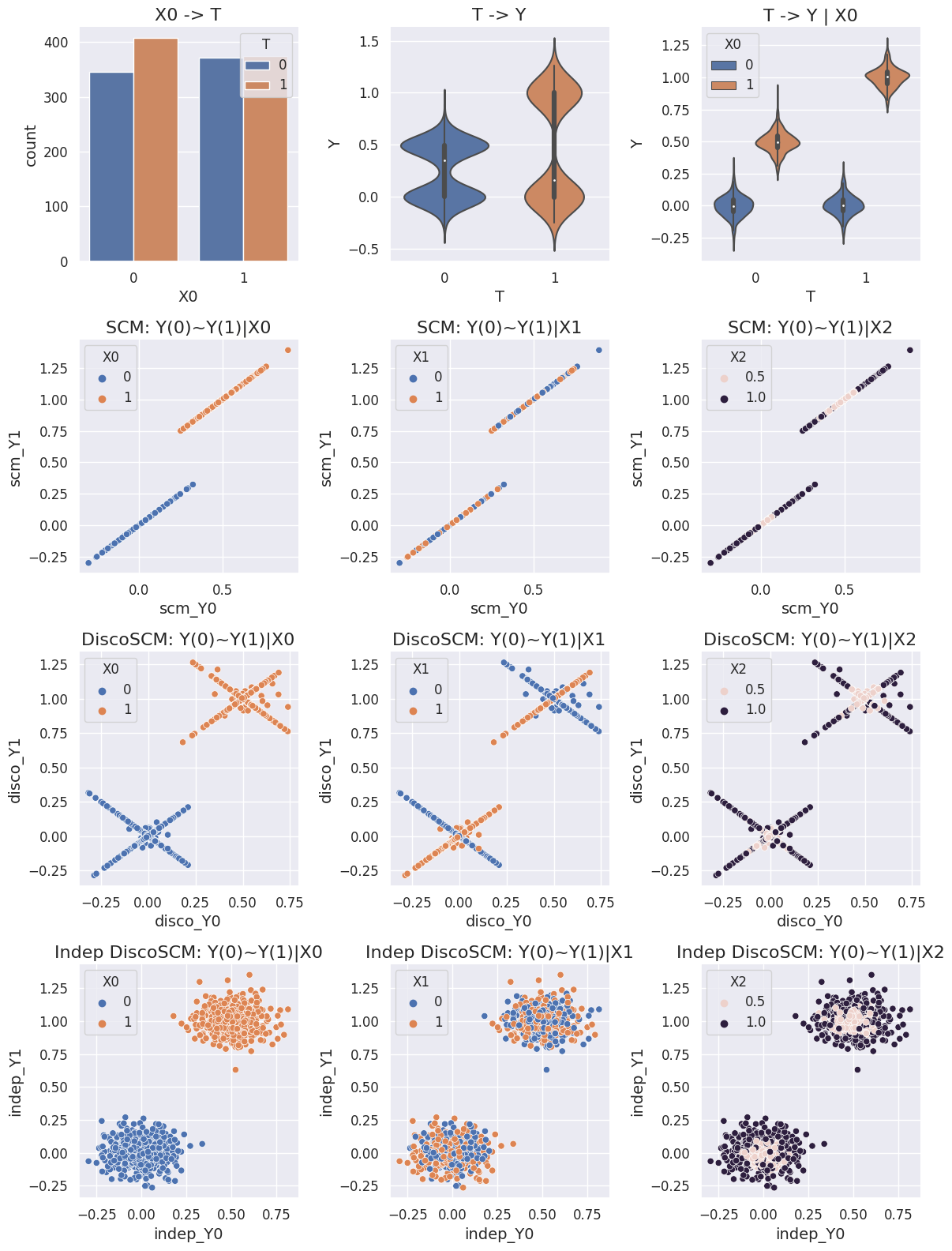}
    \caption{\small DiscoSCM with heterogeneous causal effects, counterfactual noises correlations, and  probabilities of consistency. The type of this DiscoSCM depends on the correlation pattern among counterfactual noises $E_{X_1}(t)$, see Fig. \ref{fig:noise_corr}: if the correlation coefficient is 1, it resembles an ordinary SCM; if there is some correlation, it is a general DiscoSCM with indeterminable heterogeneous counterfactuals; if the counterfactual noises are independent, it becomes a DiscoSCM where Layer 3 individual-level counterfactuals can be reduced to Layer 2 valuations. Specifically, for the counterfactual parameters $corr(Y^d_u(0), Y^d_u(1))$, the second row of Fig. \ref{fig:indepDiscoSCM} shows that it is always 1 in the SCM, while the third row reveals that its value lies between 0 and 1 in the general DiscoSCM, showing heterogeneity according to $X_1$. The fourth row of Fig. \ref{fig:indepDiscoSCM} demonstrates that this parameter is always 0 in a DiscoSCM with independent potential noise. To summarize, when the correlation between counterfactual noises is 1, as is the case in SCM, complete knowledge of structural equations is required to deducing for counterfactuals. When counterfactual noises exhibit some correlation, neither RCT or observational data can help recover related counterfactual parameters. Lastly, when counterfactual noises are independent, Layer 3 valuation can be reduced to Layer 2, typically allowing them to be learned from data.}
    \label{fig:indepDiscoSCM}
\end{figure}

The preceding example illustrates the enhanced capability of the DiscoSCM framework to represent causal information, notably the heterogeneous probabilities of counterfactual outcomes $P(y_t, y'_{t'})$ across individuals. When we assume individuals within a population have homogeneous Layer 3 valuations, we can leverage Theorem \ref{thm3_new} to further identify counterfactual parameters at the population level effectively. The theorem simplifies the calculation of these probabilities as follows:
\begin{align*}
    P(y_t, y'_{t'})& \delequal P(Y^d(t) = y, Y^d(t') = y') \\
    &= \sum_u P(Y^d_u(t) = y, Y^d_u(t') = y') P(U=u) \\
    &= \sum_u P(Y^d_u(t) = y)P(Y^d_u(t') = y') P(U=u) \\
    &= P(Y^d(t) = y)P(Y^d(t') = y') \sum_u  P(U=u) \\
    &= P(y_t) P(y'_{t'})
\end{align*}
where $P(y_t)$ and $P(y'_{t'})$ are usually identifiable with RCT data. For binary treatment and binary outcome scenarios, the PNS $P(y_t, y'_{t'})$ represents a probability of causation parameter, signifying the probability that an individual is a complier. This showcases the capability of the DiscoSCM framework to utilize assumptions of homogeneity and ICN, thereby simplifying and identifying population-level counterfactual parameters into simpler and identifiable components.

In scenarios lacking the ICN assumption, DiscoSCM might not directly identify counterfactuals; however,  it is still feasible to derive theoretical heterogeneous bounds across individuals for counterfactuals regarding probability of causation, such as $P(y_t, y'_{t'})$. These bounds, derived in the forthcoming section, are identifiable or can be learned from data, thereby enhancing our understanding of how to embrace heterogeneity in counterfactual inference.

\section{Bounds of Counterfactuals}

In recent years, a plethora of studies have set forth  many results concerning bounds for population-level counterfactuals, applying these findings to the unit selection problem \cite{li2019, li2022unit}.  The DiscoSCM framework enables an intricate examination of bounds at the individual level. \footnote{Owing to Theorem \ref{theo:population}, these individual-level theorems naturally extend to applicable (sub-)population-level implications, offering a more nuanced understanding.}. This shift from analyzing population-level effects to individual-level facilitates discussions on non-degenerative heterogeneous counterfactuals concerning the probability of causation.
Before delving deeper, the following essential lemma is presented:
\begin{lemma}
In a DiscoSCM, for any individual $u$ with a binary treatment $T$ and outcome $Y$:
\begin{equation}
\label{eq:obs_outcome}
    P(y) = P(y_t) P(t) + P(y_{t'}) P(t')
\end{equation}
\end{lemma}
\begin{proof} 
By probability formula and Theorem \ref{theo:individual}:
\begin{align*}
    P(y;u) 
    &\delequal P(Y=y|U=u) \\ 
    & = P(Y=y|T=t, U=u) P(T=t|U=u)  \\
    &\quad + P(Y=y|T=t', U=u) P(T=t'|U=u)\\
    &= P(y|t; u) P(t;u) + P(y|{t'};u) P(t';u) \\
    &= P(y_t;u) P(t;u) +  P(y_{t'};u) P(t';u)
\end{align*}
\end{proof}

The following results can then be derived :
\begin{theorem}
\label{thm:bound}
In a DiscoSCM, for any individual $u$ with an observed binary treatment $T$, outcome $Y$:\footnote{Preliminaries on probability of causation parameters PNS and PN are provided in the Appendix. \ref{sec:poc}}
\begin{eqnarray}
\small
\max \left \{
\begin{array}{cc}
0 \\
P(y_t) - P(y_{t'}) \\
P(y) - P(y_{t'}) \\
P(y_t) - P(y)\\
\end{array}
\right \}
\small
\le \text{PNS}
\label{pnslb}
\end{eqnarray}
\begin{eqnarray}
\small
\text{PNS} \le \min \left \{
\begin{array}{cc}
 P(y_t) \\
 P(y'_{t'}) \\
P(t,y) + P(t',y') \\
P(y_t) - P(y_{t'}) +\\
+ P(t, y') + P(t', y)
\end{array} 
\right \}
\label{pnsub}
\end{eqnarray}
\begin{eqnarray}
\small
\max \left \{
\begin{array}{cc}
0 \\
\frac{P(y)-P(y_{t'})}{P(t,y)}
\end{array} 
\right \}
\le \text{PN}
\label{pnlb}
\end{eqnarray}
\begin{eqnarray}
\text{PN} \le
\min \left \{
\begin{array}{cc}
1 \\
\frac{P(y'_{t'})-P(t',y')}{P(t,y)} 
\end{array}
\right \}
\label{pnub}
\end{eqnarray}
\end{theorem}
\begin{proof}
The first part of Eq. \eqref{pnslb} is trivial. The second part is
\begin{align*}
&P(y_t, y'_{t'}) \geq P(y_t) - P(y_{t'})\\
&\Leftrightarrow P(y_{t'}) \geq P(y_t) - P(y_t, y'_{t'}) \\
&\Leftrightarrow P(y_{t'}) \geq P(y_t, y_{t'})
\end{align*}
This proves the second part. The third part of Eq. \eqref{pnslb} is
\begin{align*}
&P(y_t, y'_{t'}) \geq P(y) - P(y_{t'}) \\
&\Leftrightarrow P(y_t, y'_{t'}) \geq P(y_t) P(t) + P(y_{t'}) P(t') - P(y_{t'}) \\
&\Leftrightarrow P(y_t, y'_{t'}) \geq P(t) (P(y_t) - P(y_{t'}))
\end{align*}
Using the conclusion from the second part of Eq. \eqref{pnslb} and the condition $P(t) \leq 1$, we can prove the third part. The fourth part of Eq. \eqref{pnslb} is
\begin{align*}
&P(y_t, y'_{t'}) \geq P(y_t) - P(y) \\
&\Leftrightarrow P(y_t, y'_{t'}) \geq P(y_t) - P(y_t) P(t) - P(y_{t'}) P(t') \\
&\Leftrightarrow P(y_t, y'_{t'}) \geq P(t') (P(y_t) - P(y_{t'}))
\end{align*}
This can also be proven using the conclusion from the second part of Eq. \eqref{pnslb} and the condition $P(t') \leq 1$. This concludes the proof of Eq. \eqref{pnslb} of the theorem.

The first two parts of Eq. \eqref{pnsub} is trival. The third part is:
\begin{align*}
&P(y_t, y'_{t'}) \leq P(t, y) + P(t', y') \\
&\Leftrightarrow P(y_{t'}) \leq P(y|t) P(t) + P(y'|t') P(t') \\
&\Leftrightarrow P(y_t, y'_{t'})  \leq P(t) P(y_t) + P(t') P(y'_{t'})
\end{align*}
It's evident given the first two parts of Eq. \eqref{pnsub}. The fourth part of Eq. \eqref{pnsub} is
\begin{align*}
&P(y_t, y'_{t'}) \leq P(y_t) - P(y_{t'}) + P(t, y') + P(t', y) \\
&\Leftrightarrow P(y_{t'}) \leq P(y_t) - P(y_{t'}) + P(y'|t) P(t) + P(y|t') P(t') \\
&\Leftrightarrow P(y_{t'}) \leq P(y_t) - P(y_{t'}) + P(y'_t) P(t) + P(y_{t'}) P(t') \\
&\Leftrightarrow P(y_t, y'_{t'})  \leq P(t') P(y_t) + P(t) P(y'_{t'})
\end{align*}
This can also be proven using the first two parts of Eq. \eqref{pnsub}. This concludes the proof of Eq. \eqref{pnsub} of the theorem.

The first part of Eq. \ref{pnlb} is trivial. The second part is
\begin{align*}
&P(y'_{t'}|t, y) \geq \frac{P(y) - P(y_{t'})}{P(t, y)} \\
&\Leftrightarrow P(y'_{t'}) P(t, y) \geq P(y) - P(y_{t'}) \\
&\Leftrightarrow P(y'_{t'}) P(y_t) P(t) + P(y_{t'}) \geq P(y_t) P(t) + P(y_{t'}) P(t') \\
&\Leftrightarrow P(y_{t'}) P(t) \geq P(y_t) P(t) P(y_{t'}) \\
&\Leftrightarrow 1 \geq P(y_t) 
\end{align*}
This proves the second part of Eq. \eqref{pnlb}.

The first part of Eq. \ref{pnub} is trivial. The second part is
\begin{align*}
&P(y'_{t'}|t, y) \leq \frac{P(y'_{t'}) - P({t'}, y')}{P(t, y)} \\
&\Leftrightarrow P(y'_{t'}) P(t, y) \leq P(y'_{t'}) - P({t'}, y') \\
&\Leftrightarrow P(y'_{t'}) P(y_t) P(t) \leq P(y'_{t'}) - P(y'|t') P(t') \\
&\Leftrightarrow P(y'_{t'}) P(y_t) P(t) \leq P(y'_{t'}) P(t) \\
&\Leftrightarrow 1 \leq P(y_t)
\end{align*}
This proves the second part of Eq. \eqref{pnub}.
\end{proof}
When additional structural information is available, tighter bounds can also be derived. 
\begin{theorem}
In a DiscoSCM, for any individual $u$ with a observed binary treatment $T$, outcome $Y$, and a partial mediator $Z$: 
\begin{flushleft}
\small
\begin{eqnarray}
\max\left\{
\begin{array}{c}
0,\\
P(y_t)-P(y_{t'}),\\
P(y)-P(y_{t'}),\\
P(y_t)-P(y)\\
\end{array}
\right\}\le \text{PNS}
\label{lb_mediator_plus_direct}
\end{eqnarray}
\end{flushleft}
\begin{flushleft}
\small
\begin{eqnarray}
\min\left\{
\begin{array}{c}
P(y_t),\\
P(y'_{t'}),\\
P(y,t)+P(y',t'),\\
P(y_t)-P(y_{t'}) + P(y,t')+P(y',t),\\
\sum_z \sum_{z'} \min\{P(y|z,t), P(y'|z',t')\}\\
\times \min\{P(z_t),P(z'_{t'})\}
\end{array}
\right\}\ge \text{PNS}
\label{ub_mediator_plus_direct_new}
\end{eqnarray}
\end{flushleft}
\label{thm3_new}
\end{theorem}
\begin{proof}
Given the previously established Theorem \ref{thm:bound}, it suffices to prove the following equation:

\begin{align*}
&P(y_t, y'_{t'})  \leq \sum_z \sum_{z'} \min\{P(y|z,t), P(y'|z',t')\}\times \min\{P(z_t),P(z'_{t'})\} \\ 
&\Leftrightarrow \sum_z \sum_{z'} P(y_t,y'_{t'},z_t,z'_{t'}) \\ 
&\leq \sum_z \sum_{z'} \min\{P(y|z,t), P(y'|z',t')\}\times \min\{P(z_t),P(z'_{t'})\} \\ 
&\Leftrightarrow \sum_z \sum_{z'} P(y_t,y'_{t'}|z_t,z'_{t'}) P(z_t,z'_{t'}) \\ 
&\leq \sum_z \sum_{z'} \min\{P(y|z,t), P(y'|z',t')\}\times \min\{P(z_t),P(z'_{t'})\} \\ 
&\Leftrightarrow \sum_z \sum_{z'} P(y_{t, z},y'_{t', z'}) P(z_t,z'_{t'})  \\
&\leq \sum_z \sum_{z'} \min\{P(y_{z,t}), P(y'_{z',t'})\}\times \min\{P(z_t),P(z'_{t'})\} 
\end{align*}

It thus left to prove:

\begin{enumerate}
    \item $P(y_{x, z},y'_{t', z'}) \leq \min\{P(y_{z,t}), P(y'_{z',t'})\}$, and
    \item $P(z_t,z'_{t'}) \leq \min\{P(z_t),P(z'_{t'})\}$.
\end{enumerate}

Both are evidently true by probability formula.
\end{proof}
 
In the context of the unit selection problem, individual responses are delineated into four distinct categories: compliers, always-takers, never-takers, and defiers. The associated benefits of selecting an individual from each category are represented by $\beta$, $\gamma$, $\theta$, and $\delta$, respectively. This delineation forms the basis of a benefit function for any given individual $u$, expressed as:
\begin{align}
f_u = \beta P(y_t, y'{t'}) + \gamma P(y_t, y{t'}) + \theta P(y't, y'{t'}) + \delta P(y'{t}, y{t'}),
\end{align}
where $f_u$ quantifies the expected benefit of making a counterfactual-informed decision for individual $u$. A lemma related to this benefit function states:
\begin{lemma}
Given a DiscoSCM,  for any individual $u$, the benefit function:
\begin{align}
    f_u = W_u + \sigma \text{PNS}
\end{align}
where $\text{PNS} = P(y_t, y'_{t'})$, and
$$W_u = (\gamma-\delta) P(y_t) + \delta P(y_{t'}) + \theta P(y'_{t'})$$ 
is identifiable or can be learned from data.
\end{lemma}

\begin{proof}
By properties of probability:
\begin{align*}
    &W_u - f_u(c) \\
    &=  W_u - (\beta P(y_t, y'_{t'}) + \gamma  P(y_t, y_{t'}) + \theta  P(y'_t, y'_{t'}) + \delta  P(y'_{t}, y_{t'})) \\
    &= \gamma  P(y_t, y'_{t'}) + \theta  P(y_t, y'_{t'}) - \delta ( P(y_t) - P(y_{t'}) + P(y'_{t}, y_{t'}))  -  \beta P(y_t, y'_{t'}) \\
    &= \gamma  P(y_t, y'_{t'}) + \theta  P(y_t, y'_{t'}) - \delta  P(y_t, y'_{t'}) -  \beta P(y_t, y'_{t'}) \\
    &= P(y_t, y'_{t'}) (\gamma + \theta - \delta - \beta) \\
    &= - \sigma \text{PNS}
\end{align*}  
\end{proof}
Based on this Lemma, and integrating the preceding theorems concerning the identifiable or learnable bounds for the PNS, bounds for the benefit function at the individual level can thus be obtained.

\section{Conclusion}

Answering counterfactual questions with data, termed as Layer 3 valuation, poses a significant challenge. Traditional causal modeling frameworks exhibit inherent weaknesses in modeling counterfactuals due to the degeneration problem caused by the consistency rule. However, the DiscoSCM framework, grounded in the distribution-consistency rule, has addressed these theoretical challenges by introducing appropriate assumptions.

To summarize, when the correlation between counterfactual noises is 1, as in standard SCMs, structural equations is necessary to be learned and solved for computing counterfactuals. When counterfactual noises exhibit some level of correlation, neither RCT nor observational data can guarantee the recovery of related counterfactual parameters due to the indeterminable counterfactuals issue, though bounds for these parameters can still be obtained. Lastly, when counterfactual noises are independent, Layer 3 valuations can be simplified to Layer 2, allowing for their typical identification or learning from data. In conclusion, we advocate that DiscoSCM represents a significant advancement in counterfactual inference.

\subsection{Further Discussions}

Heterogeneous causal effects have been extensively studied, yet research into areas such as heterogeneous causal graphs and heterogeneous counterfactual probabilities remains relatively scarce. We posit that adopting the DiscoSCMs framework, with its enhanced causal modeling capacity, might significantly advance research in these related areas. The DiscoSCMs framework, with its unique perspective on understanding of causality, offers new avenues for counterfactual inference that embracing heterogeneity.

\paragraph{What are counterfactuals?} 
The term  ``counterfactuals'' is marked by ambiguity, encompassing a broad spectrum of interpretations within the domain of causal inference. It's crucial to delineate and explore its meaning for comprehensive understanding. Counterfactuals can refer to events, information, queries, and variables, which are broadly used in areas like counterfactual machine learning or reasoning, covering various techniques, basically all analyses related to causality, representing a widespread academic practice. Another prevailing interpretation associates counterfactuals with quantities at Layer 3 of the PCM, formalized as Layer valuations, notable for involving at least two different worlds. Aligning with this dual interpretation, we propose an understanding that views counterfactuals as pertaining to anything about the counterfactual world, and perceiving counterfactuals as layer 3 valuations within the PCH context.

Viewing from a broader lens, statistics is about the reduction of data (R.A. Fisher, 1922), and machine learning involves learning hypotheses from data through Empirical Risk Minimization (ERM), both operating within the factual world. In contrast, causal inference distinguishes itself by operating in the counterfactual or imaginary world, with its core focused on counterfactual inference. This distinction highlights the unique domain of causal inference, setting it apart from the data-driven approaches of statistics and machine learning through its focus on exploring the "what-ifs" that lie beyond observed reality.

\paragraph{Embracing Heterogeneity with DiscoSCMs?} 
In discussing the development of the innovative DiscoSCM framework, the PO framework and SCMs serve as two pivotal references. On one hand, inspired by the PO framework's clear semantics for representing individuals with an index, DiscoSCM advances this foundation by explicitly incorporating a unit selection variable \(U\). This variable initially acts as an index, linking data points to specific individuals, with each instantiation \(U = u\) functioning similarly to pre-treatment features in the PO framework, as revealed in Section 5.3 in \cite{gong2024discoscm}. On the other hand, inspired by SCMs, DiscoSCM positions underlying causal mechanisms as primitives, deriving a tri-level hierarchy for causal information, and formalizing them as Layer valuations. Importantly, the unit selection variable \(U\) plays a central role in Layer valuations, including those at the population level (as described in Theorem \ref{theo:population}) and the individual level (as outlined in Theorem \ref{theo:individual}).

To delve into the core and intrinsic logic of DiscoSCM, four rationales are emphasized. First, the novel perspective of examining causality (population versus individual) manifests \textit{individual causality}. Second, invariant causal mechanisms across individuals ensure homogeneity in causal parameters. Third, the presence of heterogeneous individuals necessitates the use of causal embeddings for each individual in Layer valuations. Fourth, causality operates sophisticatedly on both homogeneity and heterogeneity among individuals. In summary, these rationales collectively affirm the profound insight that ``causality is invariance across heterogeneous units,'' succinctly encapsulated by the term \textit{individual causality}.

Thus, DiscoSCMs distinguish themselves from traditional causal modeling frameworks by their enhanced capacity to embrace heterogeneous counterfactuals. This capability provides new insights into traditional challenges. For instance, the interpretation of the Surrogate Paradox, as discussed in \cite{gong2024discoscm}, fundamentally pertains to a heterogeneous causal graph across individuals. Traditionally, this paradox has often been approached through the lens of unobserved confounding. However, DiscoSCM offers a more intuitive and fundamental explanation, illustrating the framework's unique ability to navigate the complexities of heterogeneity in counterfactual inference.

\bibliographystyle{ACM-Reference-Format}
\bibliography{reference}

\appendix

\section{Preliminaries on Counterfactual Inference}
\label{sec:prelim}

\subsection{Causal Frameworks Based on Consistency}

The principal frameworks in causal modeling, Potential Outcomes (PO) \cite{neyman1923application, imbens2015causal} and Structural Causal Models (SCM) \cite{pearl2009causality}, hold theoretical equivalence and are both anchored in the consistency rule \cite{angrist1996identification, cole2009consistency, pearl2011consistency}. The PO approach centers on experimental units, emphasizing individual semantics, whereas SCM is based on structural equations, from which it delineates three levels of causal information: associational, interventional, and counterfactual.

The PO framework, also known as the Rubin Causal Model\cite{holland1986statistics}, begins with a population of experiment units. There is a treatment that can take on different values for each unit. Each unit in the population is characterized by a set of potential outcomes $Y(t)$, one for each value of the treatment. Only one of these potential outcomes can be observed, namely the one corresponding to the treatment received:
\begin{equation}
\label{eq:y_obs}
Y=\sum_{t} Y(t){\bf 1}_{T=t}.
\end{equation}
This equation is a derivation of the consistency assumption. 
\begin{asmp}[\textbf{Consistency}] The potential outcome $Y(t)$ precisely matches the observed variable $Y$ given observed treatment $T=t$, i.e.,
\begin{align}
\label{assump:consist}
    T=t \Rightarrow Y(t) = Y.
\end{align}
\end{asmp}

The framework of \textit{structural causal models} (SCMs) is presented as follows.  

\begin{definition}[\textbf{Structural Causal Models} \citep{pearl2009causality}]
    \label{def:scm} 
    A structural causal model is a tuple $\langle \mathbf{U}, \mathbf{V}, \mathcal{F}\rangle$, where 
    \begin{itemize}
        \item $\mathbf{U}$ is a set of background variables, also called exogenous variables, that are determined by factors outside the model, and $P(\cdot)$ is a probability function defined over the domain of $\mathbf{U}$; 
        \item $\mathbf{V}$ is a set $\{V_1, V_2, \ldots, V_n\}$ of (endogenous) variables of interest that are determined by other variables in the model -- that is, in $\mathbf{U} \cup \mathbf{V}$;
        \item $\mathcal{F}$ is a set of functions $\{f_1, f_2, \ldots, f_n\}$ such that each $f_i$ is a mapping from (the respective domains of) $U_{i} \cup Pa_{i}$ to $V_{i}$, where $U_{i} \subseteq \*{U}$, $Pa_{i} \subseteq \mathbf{V} \setminus V_{i}$, and the entire set $\mathcal{F}$ forms a mapping from $\mathbf{U}$ to $\mathbf{V}$. That is, for $i=1,\ldots,n$, each $f_i \in \mathcal{F}$ is such that 
        $$v_i \leftarrow f_{i}(pa_{i}, u_{i}),$$ 
        i.e., it assigns a value to $V_i$ that depends on (the values of) a select set of variables in $\*U \cup \*V$.
    \end{itemize}
\end{definition}
Interventions are defined through a mathematical operator.
\begin{definition}[\textbf{Submodel-``Interventional SCM''} \citep{pearl2009causality}]
    Consider an SCM $\langle \mathbf{U}, \mathbf{V}, \mathcal{F}\rangle$, with a set of variables $\*X$  in $\*V$, and a particular realization $\*x$ of $\*X$. The $do(\*x)$ operator, representing an intervention (or action), modifies the set of structural equations $\mathcal{F}$ to $\mathcal{F}_{\*x} := \{f_{V_i} : V_i \in \*V \setminus \*X\} \cup \{f_X \leftarrow x : X \in \*X\}$ while maintaining all other elements constant. Consequently, the induced tuple $\langle \mathbf{U}, \mathbf{V}, \mathcal{F}_{\*x}\rangle$ is called as \textit{Intervential SCM} , and potential outcome $\*Y(\*x)$ (or denoted as $\*Y_{\*x}(\*u)$) is defined as the set of variables $\*Y \subseteq \*V$ in this submodel.
\end{definition}
Formally, an SCM gives valuation for associational, interventional and counterfactual quantities in the Pearl Causal Hierarchy (PCH) as follows.

\begin{definition}[\textbf{Layer Valuation} \citep{bareinboim2022pearl}]
\label{def:l3-semantics}
An SCM $\langle \*U, \*V, \mathcal{F}\rangle$ induces a family of joint distributions over potential outcomes $\*Y(\*x), \ldots, \*Z({\*w})$, for any $\*Y$, $\*Z$, $\dots$, $\*X,$ $\*W \subseteq \*V$: 
\begingroup\abovedisplayskip=0.5em\belowdisplayskip=0pt
\begin{align}\label{eq:def:l3-semantics}
    P(\*{y}_{\*{x}},\dots,\*{z}_{\*{w}}) =
\sum_{\substack{\{\*u\;\mid\;\*{Y}({\*x})=\*{y}, \dots,\; \*{Z}({\*w})=\*z\}}}
    P(\*u).
\end{align}
\endgroup
is referred to as Layer 3 valuation. In the specific case involving only one intervention, e.g., $do(\*x)$:
\begingroup\abovedisplayskip=0.5em\belowdisplayskip=0pt
\begin{align}
    \label{eq:def:l2-semantics}
    P({\*y}_{\*x}) = \sum_{\{\*u \mid {\*Y}({\*x})={\*y}\}}
    P(\*u),
\end{align}
\endgroup
is referred to as Layer 2 valuation. The even more specialized case when $\*X$ is empty:
\begingroup\abovedisplayskip=0.5em\belowdisplayskip=0pt
\begin{align}
    \label{eq:def:l1-semantics}
    P({\*y}) = 
    \sum_{\{\*u \mid {\*Y}={\*y}\}}
    P(\*u).
\end{align}
\endgroup
is referred to as Layer 1 valuation. Here, $\*y$ and $\*z$ represent the observed outcomes, $\*x$ and $\*w$ the observed treatments, $\*u$ the noise instantiation, and we denote $\*y_{\*x}$ and $\*z_{\*w}$ as the realization of their corresponding potential outcomes. 
\end{definition}

In the case of recursive SCMs, the $do$-calculus can be employed to completely identify all Layer 2 expressions \cite{pearl1995causal, huang2012pearl}. However, calculating counterfactuals at Layer 3 is generally far more challenging compared to Layers 1 and 2. This is because it essentially requires modeling the joint distribution of potential outcomes, such as the potential outcomes with and without aspirin. Unfortunately, we often lack access to the underlying causal mechanisms and only have observed traces of them. This limitation leads to the practical use of Eq. \eqref{eq:def:l3-semantics} for computing counterfactuals being quite restricted.

\subsection{Probability of Causation}
\label{sec:poc}

The probability of causation and its related parameters can be addressed by counterfactual logical \cite{pearl1999probabilities, pearl2009causality}, three prominent concepts of which are formulated in the following  :
\begin{definition}[Probability of necessity (PN)]
Let $X$ and $Y$ be two binary variables in a causal model $M$, let $x$ and $y$ stand for the propositions $X=true$ and $Y=true$, respectively, and $x'$ and $y'$ for their complements. The probability of necessity is defined as the expression:
\begin{eqnarray}
\text{PN} & \delequal & P(Y_{x'}=false|X=true,Y=true)\nonumber \\ 
& \delequal & P(y'_{x'}|x,y)
\label{def:pn}
\end{eqnarray}
\end{definition}
\par
In other words, PN stands for the probability that event $y$ would not have occurred in the absence of event $x$, given that $x$ and $y$ did in fact occur. This counterfactual notion is used frequently in lawsuits, where legal responsibility is at the center of contention.
\begin{definition}[Probability of sufficiency (PS) ] 
\begin{eqnarray}
\text{PS}\delequal P(y_x|y',x')
\label{def:ps}
\end{eqnarray}
\end{definition}
\begin{definition}[Probability of necessity and sufficiency (PNS)]
\begin{eqnarray}
\text{PNS}\delequal P(y_x,y'_{x'})
\label{def:pns}
\end{eqnarray}
\end{definition}
PNS stands for the probability that $y$ would respond to $x$ both ways, and therefore measures both the sufficiency and necessity of $x$ to produce $y$. Tian and Pearl \cite{tian2000} provide tight bounds for PNS, PN, and PS without a causal diagram:
\begin{eqnarray}
\max \left \{
\begin{array}{cc}
0 \\
P(y_x) - P(y_{x'}) \\
P(y) - P(y_{x'}) \\
P(y_x) - P(y)\\
\end{array}
\right \}
\le \text{PNS}
\label{pnslb_old}
\end{eqnarray}
\begin{eqnarray}
\text{PNS} \le \min \left \{
\begin{array}{cc}
 P(y_x) \\
 P(y'_{x'}) \\
P(x,y) + P(x',y') \\
P(y_x) - P(y_{x'}) +\\
+ P(x, y') + P(x', y)
\end{array} 
\right \}
\label{pnsub_old}
\end{eqnarray}
\begin{eqnarray}
\max \left \{
\begin{array}{cc}
0 \\
\frac{P(y)-P(y_{x'})}{P(x,y)}
\end{array} 
\right \}
\le \text{PN}
\label{pnlb_old}
\end{eqnarray}
\begin{eqnarray}
\text{PN} \le
\min \left \{
\begin{array}{cc}
1 \\
\frac{P(y'_{x'})-P(x',y')}{P(x,y)} 
\end{array}
\right \}
\label{pnub_old}
\end{eqnarray}
In fact, when further structural information is available, we can obtain even tighter bounds for those parameters, as highlighted by recent research \cite{li2022unit}.

\end{CJK}
\end{document}